%% file: main.tex
\documentclass{article}

\usepackage{microtype}
\usepackage{graphicx}
\usepackage{subfigure}
\usepackage{booktabs} 
\usepackage{multirow}
\usepackage{stfloats}

\usepackage{hyperref}

\usepackage[accepted]{icml2025}

\usepackage{amsmath, amsthm}
\usepackage{thmtools}
\usepackage{amssymb}
\usepackage{mathtools}
\usepackage{amsthm}
\usepackage{paralist}
\usepackage{mathrsfs}

\usepackage[capitalize,noabbrev]{cleveref}

\theoremstyle{plain}
\newtheorem{theorem}{Theorem}[section]

\newtheorem{lemma}[theorem]{Lemma}

\theoremstyle{definition}

\theoremstyle{remark}
\newtheorem{remark}[theorem]{Remark}

\usepackage[textsize=tiny]{todonotes}

\input{macro}

\icmltitlerunning{Qualitative Analysis of $\omega$-Regular Objectives on Robust MDPs}
\begin{document}

\twocolumn[
\icmltitle{Qualitative Analysis of $\omega$-Regular Objectives on Robust MDPs}




\begin{icmlauthorlist}
\icmlauthor{Ali Asadi}{ist}
\icmlauthor{Krishnendu Chatterjee}{ist}
\icmlauthor{Ehsan Kafshdar Goharshady}{ist}
\icmlauthor{Mehrdad Karrabi}{ist}
\icmlauthor{Ali Shafiee}{sharif}
\end{icmlauthorlist}

\icmlaffiliation{ist}{Institute of Science and Technology Austria (ISTA), Klosterneuburg, Austria}
\icmlaffiliation{sharif}{Sharif University of Technology, Tehran, Iran}

\icmlcorrespondingauthor{Ali Asadi}{aasadi@ista.ac.at}
\icmlcorrespondingauthor{Ehsan Kafshdar Goharshady}{egoharsh@ista.ac.at}
\icmlcorrespondingauthor{Mehrdad Karrabi}{mkarrabi@ista.ac.at}

\icmlkeywords{Machine Learning, ICML}

\vskip 0.3in
]



\printAffiliationsAndNotice{}  

\begin{abstract}
Robust Markov Decision Processes (RMDPs) generalize classical MDPs that consider uncertainties in transition probabilities by defining a set of possible transition functions. An objective is a set of runs (or infinite trajectories) of the RMDP, and the value for an objective is the maximal probability that the agent can guarantee against the adversarial environment. We consider (a)~reachability objectives, where given a target set of states, the goal is to eventually arrive at one of them; and (b)~parity objectives, which are a canonical representation for $\omega$-regular objectives. The qualitative analysis problem asks whether the objective can be ensured with probability 1. 

In this work, we study the qualitative problem for reachability and parity objectives on RMDPs without making any assumption over the structures of the RMDPs, e.g., unichain or aperiodic. Our contributions are twofold. We first present efficient algorithms with oracle access to uncertainty sets that solve qualitative problems of reachability and parity objectives. We then report experimental results demonstrating the effectiveness of our oracle-based approach on classical RMDP examples from the literature scaling up to thousands of states.
\end{abstract}

\input{intro}
\input{prelims}
\input{algo}
\input{experiments}

\input{conclu}

\section*{Impact Statement}
This paper presents work whose goal is to advance the field of Machine Learning. There are many potential societal consequences of our work, none which we feel must be specifically highlighted here.

\section*{Acknowledgements}
This work was supported by ERC CoG 863818 (ForM-SMArt) and Austrian Science Fund (FWF) 10.55776/COE12. We also thank Hossein Zakerinia for his helpful feedback.

\bibliography{main}
\bibliographystyle{icml2025}
\newpage\clearpage
\appendix
\onecolumn
\input{appendix}

\end{document}

%% file: macro.tex
\newcommand{\mdp}{M}
\newcommand{\rmdp}{\mathcal{M}}
\newcommand{\states}{\mathcal{S}}
\newcommand{\actions}{\mathcal{A}}
\newcommand{\init}{\mathit{init}}
\newcommand{\sinit}{s_\init}
\newcommand{\trans}{\delta}
\newcommand{\agentpol}{\sigma}
\newcommand{\envpol}{\rho}
\newcommand{\uncert}{\mathcal{P}}
\newcommand{\run}{\pi}
\newcommand{\runs}{\Pi}
\newcommand{\prob}{\mathbb{P}}
\newcommand{\obj}{\mathcal{O}}
\newcommand{\reach}{\mathit{Reach}}
\newcommand{\parity}{\mathit{Parity}}
\newcommand{\buchi}{\mathit{B\text{\textit{\"u}}chi}}
\newcommand{\coparity}{\overline{\mathit{Parity}}}
\newcommand{\N}{\mathbb{N}}
\newcommand{\R}{\mathbb{R}}
\newcommand{\val}{\mathit{Val}}

\newcommand{\force}{\mathit{force}}
\newcommand{\true}{\mathit{true}}
\newcommand{\false}{\mathit{false}}
\newcommand{\agent}{\mathscr{A}}
\newcommand{\env}{\mathscr{E}}
\newcommand{\PA}{\mathit{PAttr}}
\newcommand{\argmax}{\mathit{argmax}}
\newcommand{\asreach}{\mathit{ASReach}}
\newcommand{\asparity}{\mathit{ASParity}}
\newcommand{\easparity}{\mathit{EffASParity}}
\newcommand{\size}{\mathit{MS}}
\newcommand{\supp}{\text{supp}}

%% file: intro.tex
\section{Introduction} \label{sec:intro}

\paragraph{Robust Markov Decision Processes.}
Markov Decision Processes (MDPs) are the main framework for reasoning, decision making and planning under uncertainty \cite{Puterman94}. Solving an MDP focuses on finding a policy for the agent such that a specific trajectory-related (temporal) objective is satisfied with high probability, or the expectation of a reward function is maximized. However, the stochasticity in MDPs is usually estimated by sampling trajectories from the real-world model, resulting in that the exact transition probabilities are not known \cite{Walley96,WangZ21,WangZ22c}. Classical MDP analysis algorithms do not consider such uncertainties. To this end, Robust MDPs (RMDPs) are introduced, by considering a restricted set of transition functions and assuming that the original transition function is included~\cite{NilimG03, Iyengar05}. Then the main goal of RMDP analysis is to find an agent policy that maximizes the worst-case probability of a temporal property satisfaction or the worst-case expected reward with respect to all possible choices of transition function from the uncertainty set. This way, a level of {\em robustness} is guaranteed for the proposed policy. 

\paragraph{Quantitative vs Logical Objectives.} 
Most of the previous works on RMDPs, consider discounted sum or long-run average maximization objectives \cite{ChatterjeeGK0Z24,HoPW21,meggendorfer2024,WangVAPZ23,NilimG03,WangHP23}. These objectives are referred to as {\em quantitative} objectives since they provide performance-related guarantees on the model. However, such objectives do not consider the logical correctness of the policies which is crucial in safety-critical applications. For example, in autonomous driving, although it is important to minimize the amount of energy consumed by the vehicle, it is more important to avoid collisions (a logical safety property). For such correctness properties, the goal is to obtain policies that ensure that the property is satisfied with high probability. Logical frameworks such as linear temporal logic (LTL), or the more general class of $\omega$-regular objectives are used to model these properties. A canonical way of indicating $\omega$-regular properties is through parity objectives which is a sound framework for expressing reachability, safety, and progress conditions on finite-state models \cite{handbookmc2018}.

\paragraph{Quantitative vs Qualitative Analysis.}
Quantitative analysis ensures that the probability of satisfying an objective is above a given threshold $\lambda < 1$, whereas qualitative analysis focuses on ensuring the desired objective is satisfied with probability 1. The qualitative analysis is of great importance as in several applications where it is required that the correct behavior happens with probability 1. For example, in the analysis of randomized embedded schedulers, an important problem 
is whether every thread progresses with probability 1~\cite{Baruah92,Chatterjee13}. On the other hand, even in applications where it might be sufficient that the correct behavior happens with probability at least $\lambda < 1$,
choosing an appropriate threshold $\lambda$ can be challenging, due to modeling simplifications and imprecisions, e.g., in the analysis of randomized distributed algorithms it is common to ensure correctness with probability 1~\cite{Pogosyants00}. Besides its importance in practical applications, almost-sure convergence, like convergence in expectation, is a fundamental concept in probability theory, and provides the strongest probabilistic guarantee~\cite{Durrett19}.

\paragraph{Contributions.}
The above motivates the study of RMDPs with parity objectives and their qualitative analysis, which is the focus of this work. Along with parity we also consider the important special case of reachability objectives. Reachability refers to reaching a set of target states, which is a basic and fundamental objective. To this end, our contributions can be summarized as follows:
\begin{compactenum}
    \item We present algorithms based on specific oracle access to the uncertainty set (see Section \ref{sec:algo}) for the qualitative analysis of reachability and parity objectives. We show how these oracles are computed efficiently for well-known uncertainty set classes (see Appendix~\ref{app:oracles}).  
    \item We implemented a prototype of our algorithms and conducted experiments showing their applicability and efficiency on benchmarks motivated by the literature. We compare the performance of our reachability algorithm with the state-of-the-art long-run average approach of \cite{ChatterjeeGK0Z24} (see Section \ref{sec:experiments}). 
\end{compactenum}
Our approach is inspired by algorithms from the stochastic games literature \cite{ChatterjeeH12}, where both players have finite action space. However, extending the methods from stochastic games to RMDPs where the environment's action space is infinite is a challenge that we overcome in our work. 

\paragraph{Technical Novelty.} We discuss the technical novelties of our work as follows:
\begin{compactenum}
    \item In contrast to previous works that consider special classes of RMDPs, e.g. unichain and aperiodic \cite{WangHP23}, our approach does not assume any structural property on the model.
    \item Even in RMDPs with special classes of uncertainty, e.g. polytopic, $L_1$, or $L_\infty$, the reduction of \cite{ChatterjeeGK0Z24} to stochastic games leads to an exponential blow-up of the state space. This results in an exponential-time algorithm for analysis of reachability and parity objectives. In contrast, our approach runs in polynomial time for reachability and parity objectives with constantly many priorities, and quasi-polynomial for the general case of parity objectives. 
\end{compactenum}

\paragraph{Related Works on Robust MDPs.} Robust MDPs have been studied extensively in recent years. However, most of the literature on RMDPs considers discounted sum objectives~\cite{HoPW21,WangVAPZ23,HoPW22,BehzadianPH21} where the goal is to find a policy that maximizes the worst-case expected sum of discounted rewards seen during a run of the RMDP. Such objectives gain their motivation from classic reinforcement learning schemes \cite{RLSutton}. More recently, the community is considering long-run average (a.k.a. mean-payoff) objectives \cite{ChatterjeeGK0Z24,WangVAPZ24,Grand-ClementP23}. It is noteworthy that given an RMDP, the qualitative analysis of reachability objectives can be reduced to finding the long-run average value of an RMDP of the same size. Several works have considered specific classes of RMDPs, e.g. interval MDPs \cite{ChatterjeeSH08,TewariB07}, RMDPs with convex uncertainties \cite{clementPV23,PuggelliLSS13}, Markov Chains with uncertainties \cite{SenVA06} and parametric MDPs \cite{WinklerJPK19}. See \cite{Suilen2025} for a comprehensive survey.

\paragraph{Related Works on $\omega$-regular and Qualitative Analysis.}
$\omega$-regular and other logical objectives have been studied extensively in the context of reinforcement learning and planning \cite{HasanbeigKAKPL19,JothimuruganBBA21,SvobodaBC24,HauDGP23}. 
For parity objectives, Zielonka's algorithm \cite{Zielonka98} finds the winning set of a non-stochastic two-player parity game. Later on \cite{Parys19} modified Zielonka's approach so that it runs in quasi-polynomial runtime. In \cite{ChatterjeeJH03} authors show a reduction from stochastic parity games to non-stochastic parity games. \cite{ChatterjeeH06} proposes a randomized sub-exponential algorithm for solving stochastic parity games. \cite{ChatterjeeH12} gives a survey on analysis of $\omega$-regular objectives over two-player stochastic games. 

%% file: prelims.tex
\section{Model Definition and Preliminaries} \label{sec:prelims}
For the rest of this paper, we denote the set of all probability distributions defined over the finite set $S$ by $\Delta(S)$. Moreover, given a distribution $P$ over  $S$, we denote the support of $P$ by $\supp(P)$, the probability of sampling $B \subseteq S$ from $P$ by $P[B]$, and if $P[B]=1$, the restriction of $P$ to $B$ by $P_{|B}$. Finally, we denote the set of natural numbers as $\N$, Real numbers by $\R$ and the power set of a set $S$ as $2^S$.

\paragraph{MDPs.} A Markov Decision Process \cite{Puterman94} is defined as $\mdp = (\states, \actions, \trans)$, where $\states$ is a finite set of states, $\actions$ is a finite set of actions, and $\trans\colon \states \times \actions \rightharpoonup\Delta(\states)$ is a (partial) stochastic transition relation. We denote by $\actions(s)$ the {\em available} actions at state $s$.

The semantics of $\mdp$ are defined with respect to a policy $\agentpol\colon (\states \times \actions)^* \times \states \to \Delta(\actions)$. Given a policy $\agentpol$ for $\mdp$, the MDP starts at state $s_0$ and in the $i$-th turn, based on the {\em history} $h_i=s_0, a_0, \dots s_{i-1},a_{i-1}$ and the current state $s_i$, it samples an action $a_i$ from distribution $\agentpol(h_i,s_i)$ (over available actions $\actions(s_i)$) and evolves to state $s_{i+1}$ with probability $\trans(s_i,a_i)[s_{i+1}]$. 

\paragraph{Robust MDPs.} Robust Markov Decision Processes (RMDPs) \cite{NilimG03} extend MDPs by considering uncertainty in the transition relation. Formally, an RMDP $\rmdp$ is a tuple $(\states, \actions, \uncert)$ where $\states$ and $\actions$ are the same as in the definition of MDPs, and $\uncert \colon \states \times \actions \rightharpoonup 2^{\Delta(\states)}$ is a (partial) {\em uncertainty set}. An entry $\uncert(s,a)$ of the uncertainty set can be characterized by constraints over real-valued variables $x_1, \dots, x_{|\states|}$ representing entries of distributions over $\states$. We assume that $\uncert(s,a)$ is compact and computable for every $s,a$. Later on, in section~\ref{sec:algo}, we assume a certain kind of oracle is provided for $\uncert$.

The semantics of RMDPs are defined with respect to two policies: (i) an agent policy $\agentpol \colon (\states \times \actions)^* \times \states \to \Delta(\actions)$, and (ii) an adversarial policy $\envpol \colon (\states \times \actions)^* \times \states \times \actions \to \Delta(\states)$  for the environment, where for each $h \in (\states \times \actions)^*$, $s \in \states$ and $a \in \actions(s)$, it must hold that $\envpol(h,s,a) \in \uncert(s,a)$. Given $\agentpol$ and $\envpol$, the RMDP starts at $s_0$ and in the $i$-th step, given the history $h_i = s_0, a_0, \dots, s_{i-1}, a_{i-1}$ and the current state $s_i$, the agent samples an action $a_i$ from $\agentpol(h_i,s_i)$, then the RMDP evolves to state $s_{i+1}$ with probability $\envpol(h,s_i,a_i)[s_{i+1}]$. Intuitively, at each turn, the environment observes the action chosen by the agent and chooses a transition distribution from the uncertainty set for the RMDP to evolve accordingly. The resulting sequence $\run = s_0,a_0, s_1,a_1, \dots$ is called a {\em run} of $\rmdp$. We denote the set of all runs of $\rmdp$ by $\runs_\rmdp$.

We denote by $\prob^{\agentpol,\envpol}_{\rmdp}(s_0)$ the probability measure produced by the policies $\agentpol, \envpol$ over the runs of $\rmdp$ that start at $s_0$. This probability space is generated using cylinder sets, the details are standard and we refer the interested reader to \cite{Baier2008} for more detailed discussions.

\paragraph{Operations over RMDPs.} Given a set of states $B \subseteq \states$, the RMDP $\rmdp \setminus B$ is defined as a tuple $(\states \setminus B, \actions', \uncert')$, where (i) for every $s \in \states \setminus B$, available actions $\actions'(s)$ only contains those actions from $\actions(s)$ that cannot cause $B$ being reached in one step, i.e., $\actions'(s) = \{a \in \actions(s) \mid \forall \trans \in \uncert(s,a) \colon \trans[B] = 0 \}$, and (ii) $\uncert'$ is the same as $\uncert$ but restricted to actions in $\actions'$ and projected on $\states \setminus B$. Furthermore, the RMDP $\rmdp_{|B}$ is defined as a tuple $(B, \actions', \uncert'$), where (i) for every $s \in B$, available actions $\actions'(s)$ only contains those actions that can reach $B$ in one step with probability 1, i.e., $\actions'(s) = \{a \in \actions(s) \mid \exists \trans \in \uncert(s,a) \colon \trans[B] = 1 \}$, and (ii) $\uncert'(s,a) = \{\trans_{|B} \mid \trans \in \uncert(s,a) \wedge \trans[B]=1\}$.

\paragraph{Policy Types.} An agent policy $\agentpol$ is called {\em memoryless} (or {\em stationary}) if it does not depend on the history, i.e. if $\agentpol(h_1,s) = \agentpol(h_2,s)$ for every $h_1,h_2 \in (\states \times \actions)^*$. Moreover, $\agentpol$ is {\em pure} if it is deterministic, i.e. for each history $h \in (\states \times \actions)^*$ and each $s \in \states$, it holds that $|\sup(\agentpol(h,s))|=1$. Definitions of memoryless and deterministic environment policies are analogous.

\paragraph{Remark.} The RMDPs considered in this work are referred to as {\em $(s,a)$-rectangular} RMDPs since the environment sees both the current state and the action chosen by the agent before choosing the transition distribution. 

\paragraph{Objectives and Values.} Given an RMDP $\rmdp$, an objective $\obj$ is a set of runs of $\rmdp$. For the rest of this paper, we consider two types of objectives:
\begin{compactenum}
    \item (Reachability) Given a set $T \subseteq \states$, the reachability objective $\reach(T)$ is defined as the set of all runs that eventually reach a state in $T$: 
    \[
    \reach(T) = \{s_0,a_0, s_1,a_1, \dots \in \runs_\rmdp | \exists i: s_i \in T \}
    \]
    \item (Parity) Let $c\colon \states \to \{0, 1, \ldots, d\}$ be a function assigning priorities to states of $\rmdp$. For every run $\run \in \runs_\rmdp$, let $\inf(\pi)$ be the set of states visited infinitely often in $\run$. Then a run $\run$ is included in $\parity(c)$, if the maximum priority visited infinitely often in $\run$ is {\em even}:
    \[
    \parity(c) = \{\run \in \runs_\rmdp | \max\left\{c\left(\inf(\pi)\right)\right\} \text{ is even}\}.
    \]
\end{compactenum}

Note that the environment's version of the parity objective can be defined similarly, by replacing {\em even} with {\em odd}. 

The value of an agent policy $\agentpol$ with respect to an objective~$\obj$, given an initial state $s_0$ is defined as 
\[\val^{s_0,\agentpol}_\rmdp(\obj) = \inf_{\envpol} \prob^{\agentpol,\envpol}_\rmdp(s_0)[\obj]\]
Generally, the agent's goal is to find a policy that maximizes the value with respect to a specific objective $\obj$. To this end, the value of the RMDP $\rmdp$ given an initial state $s_0$ is defined as follows:
\[
\val^{s_0}_\rmdp(\obj) = \sup_{\agentpol} \val^{s_0,\agentpol}_\rmdp(\obj) = \sup_{\agentpol} \inf_{\envpol} \prob^{\agentpol,\envpol}_\rmdp(s_0)[\obj]
\]

\paragraph{Problem Statement.} In the rest of this paper, we develop algorithms for the following two problems, given an RMDP $\rmdp = (\states, \actions, \uncert)$, an initial state $s_0 \in \states$ and a reachability or parity objective $\obj$:
\begin{compactenum}
    \item (Limit-Sure Analysis) Decide whether $\val^{s_0}_\rmdp(\obj) = 1$.
    \item (Almost-Sure Analysis) Decide whether there exists an agent policy $\agentpol$ that guarantees $\val^{s_0,\agentpol}_\rmdp(\obj) = 1$.
\end{compactenum}

A positive answer to the almost-sure analysis problem implies a positive answer to the limit-sure analysis problem, however, the reverse does not hold in general. We will show that for RMDPs with parity objectives, the answers to the limit-sure and almost-sure analysis coincide.


%% file: algo.tex
\section{Algorithm} \label{sec:algo}
In this section, we establish our algorithms for proving almost-sure reachability and parity objectives given an RMDP $\rmdp$. To this end, we first introduce two oracle functions $\force^\rmdp_\agent$ and $\force^\rmdp_\env$, both of type $\states \times 2^\states \to \{\true,\false\}$, which we assume are provided together with the uncertainty set. 

\subsection{Oracles}
Given an RMDP $\rmdp = (\states, \actions, \uncert)$ we assume that two Boolean functions (called {\em oracles}) $\force^\rmdp_\agent,\force^\rmdp_\env\colon \states \times 2^\states \to \{\true,\false\}$ are given as part of the model description. Informally speaking, $\force^\rmdp_\agent(s,B)$ for $s\in \states$ and $B \subseteq \states$ indicates whether when the RMDP is at state $s$, the agent can force $\rmdp$ to reach $B$ with a non-zero probability in the next step, i.e., whether she can choose an action $a \in \actions(s)$ such that the set $B$ would be reached in the next step with a non-zero probability, no matter what transition function the environment chooses from $\uncert(s,a)$. 
\[
\force^\rmdp_\agent(s,B) = \begin{cases}
    \true & \exists a \in \actions. \forall \trans \in \uncert(s,a). \trans[B]>0 \\
    \false & \textit{otherwise}
\end{cases}
\]
Similarly, $\force^\rmdp_\env(s,B)$ indicates whether whenever the RMDP is at state $s$, despite the action chosen by the agent, the environment can choose a transition function so that $B$ is reached with a non-zero probability. 
\[
\force^\rmdp_\env(s,B) = \begin{cases}
    \true & \forall a \in \actions .\exists \trans \in \uncert(s,a). \trans[B]>0 \\
    \false & \textit{otherwise}
\end{cases}
\]

\begin{figure}
    \centering
    \includegraphics[width=\linewidth,trim={0 0 0 0},clip]{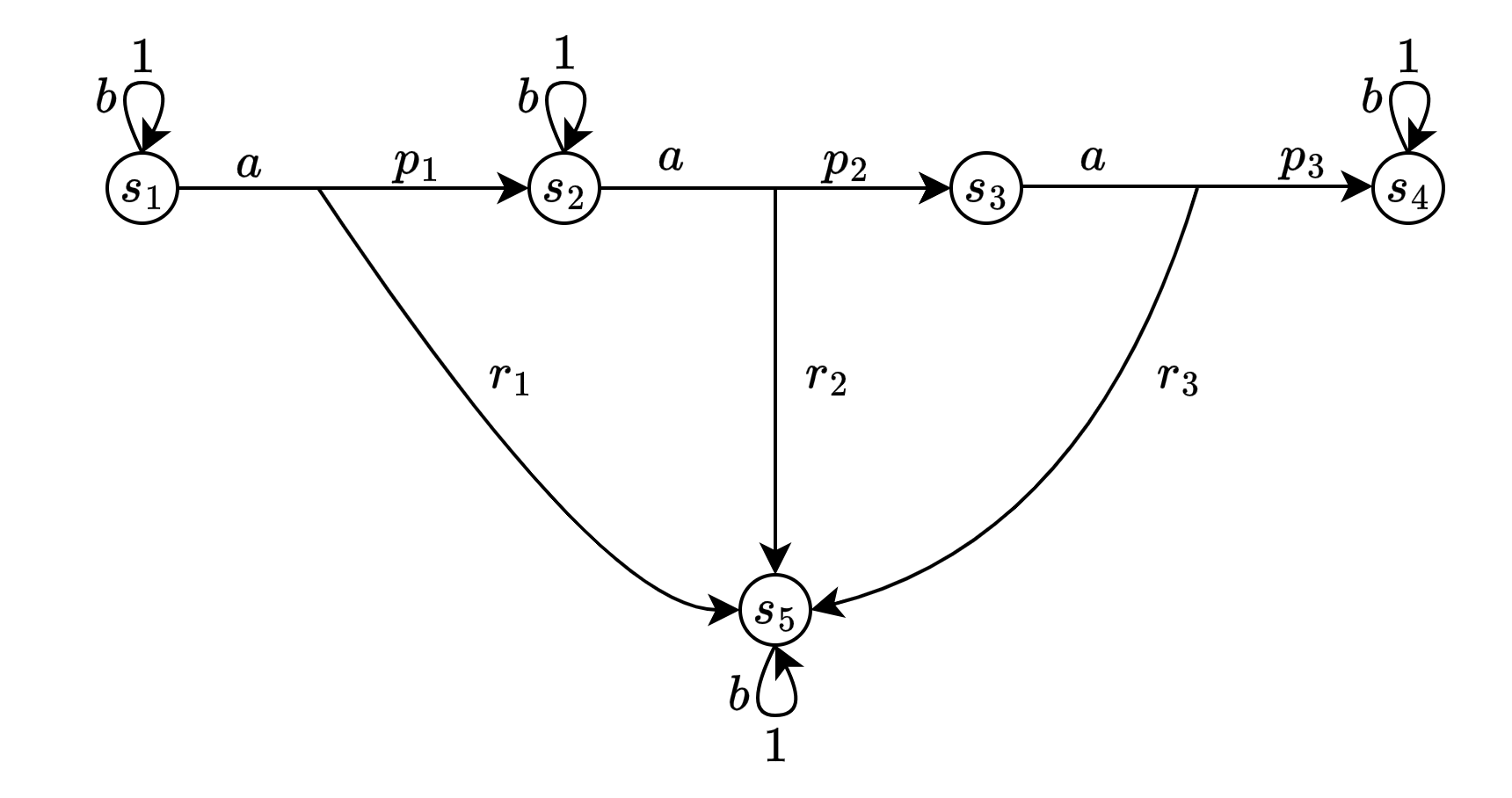}
    \vspace{-0.5cm}
    \caption{Running example for our algorithms.}
    \label{fig:reach-example}
\end{figure}

\paragraph{Example.} Consider the RMDP in Figure \ref{fig:reach-example} where action $b$ can only be applied to $s_1, s_2$, $s_4$, and $s_5$ without uncertainty, i.e., the RMDP stays in the same state when action $b$ is taken with probability 1. On the other hand, we have the following uncertainties for action $a$ over $s_1,s_2$ and $s_3$:
\[
\begin{split}
    \uncert(s_i,a) = \{ \uncert \in \Delta(\states) \mid &~ p_{i}+r_{i}=1 \\
    \wedge & \| (p_{i},r_{i}) - (\frac{1}{2},\frac{1}{2}) \|_2 \leq 0.2\}
\end{split}
\]
where $p_i$ is the probability of transitioning to $s_{i+1}$ and $r_i$ is the probability of reaching $s_5$. Then, $\force^\rmdp_\agent(s_1,\{s_1\})$, $\force^\rmdp_\agent(s_2,\{s_5\})$ and $\force^\rmdp_\env(s_3,\{s_4\})$ are all $\true$, while $\force^\rmdp_\env(s_1,\{s_5\})$ and $\force^\rmdp_\env(s_2,\{s_3\})$ are $\false$.

\paragraph{Practicality of Oracles.} Efficient (PTIME computable) oracles can be designed for real-world uncertainty set classes such as $L_1, L_2, \dots, L_\infty$, and linearly defined (polytopic) uncertainty sets. In Appendix \ref{app:oracles} we show efficient oracles for each of these classes of uncertainty sets. 

In what follows, we assume such oracles are given as part of the model presentation for $\rmdp$ and its subRMDPs with respect to the operations described in Section \ref{sec:prelims}. We provide algorithms for proving almost-sure reachability and parity specifications given access to such oracles and will then analyze them by the number of oracle calls that they make.


\paragraph{Positive Attractors.} By utilizing the oracle $\force^\rmdp_\agent$, we define the {\em Positive Attractor} procedure $\PA_\agent$ for the agent as in Algorithm \ref{alg:pattr-agent}. The key idea is that $T_i$ contains all the states where the agent can reach $T$ from them in less than or equal to $i$ steps, with a positive probability. The analogous positive attractor $\PA_\env$ for the environment is defined by replacing $\force^\rmdp_\agent$ with $\force^\rmdp_\env$ in Algorithm \ref{alg:pattr-agent} (See Appendix~\ref{app:pattre}). The following lemma, proved in Appendix \ref{app:lemm:attr}, characterizes the essential properties of $\PA_\agent(\rmdp,T)$ and $\PA_\env(\rmdp,T)$ for any $T \subseteq \states$:

\begin{restatable}{lemma}{lemmapattr}
  \label{lem:attr}
    Given an RMDP $\rmdp = (\states, \actions, \uncert)$ with oracles $\force^\rmdp_\agent$ and $\force^\rmdp_\env$ and a target set $T \subseteq \states$, the following assertions holds:
    \begin{compactenum}
        \item (Correctness of $\PA_\agent$) $s \in \PA_\agent(\rmdp,T)$ iff there exists a pure memoryless agent policy~$\agentpol$ such that for all environment policies $\envpol$ we have $\prob^{\agentpol,\envpol}_{\rmdp}(s)[\reach(T)]>0$; 
        \item (Correctness of $\PA_\env$) $s \in \PA_\env(\rmdp,T)$ iff there exists a pure memoryless environment policy~$\envpol$ such that for all agent policies $\agentpol$, we have $\prob^{\agentpol,\envpol}_{\rmdp}(s)[\reach(T)]>0$; and
        \item (Oracle Complexity) $\PA_\agent(\rmdp,T)$ and $\PA_\env(\rmdp,T)$ always terminate with $O(|\states|^2)$ oracle calls.
    \end{compactenum}
\end{restatable} 

In other words, $\PA_\agent(\rmdp,T)$ contains all states $s$ of $\rmdp$ where the agent has a policy to enforce reaching $T$ with a non-zero probability starting from $s$, despite the policy chosen by the environment. The set $\PA_\env(\rmdp,T)$ has the same meaning but for the environment. 

\begin{algorithm}[tb]
   \caption{$\PA_\agent(\rmdp,T)$ Positive Attractor for the agent}
   \label{alg:pattr-agent}
\begin{algorithmic}[1]
   \STATE {\bfseries Input:} RMDP $\rmdp$ and Target set $T$.
   \STATE Initialize $T_0 = T$, $i=0$.
   \REPEAT
   \STATE $i = i+1$
   \STATE $T_{i} = T_{i-1} \cup \{s \in \states | \force^\rmdp_\agent(s,T_{i-1})\}$
   \UNTIL{$T_i = T_{i-1}$}
   \STATE \textbf{return} $T_i$
\end{algorithmic}
\end{algorithm}

\subsection{Reachability}

This section introduces our algorithm for solving the almost-sure reachability problem in RMDPs. Algorithm \ref{alg:reach} shows the outline of our approach. The complete proofs are provided in Appendix~\ref{app:lem:reachisas,thm:reach}.

Given an RMDP $\rmdp$ with state-set $\states$, and a target set $T \subseteq \states$, the algorithm proceeds in steps, removing several states in each step. Each step first computes the set $B = \states \setminus \PA_\agent(\rmdp, T)$. These are states of the RMDP where the environment can employ a policy to never reach $T$ from them. If $B=\varnothing$, then no matter what policy the environment chooses, every state has a non-zero probability of reaching the target states. Lemma~\ref{lem:reachisas} shows that in such cases, the target $T$ can be reached almost-surely from every state of the RMDP.

\begin{restatable}{lemma}{lemmareach}
  \label{lem:reachisas}
    Let $\rmdp=(\states,\actions, \uncert)$ be an RMDP and $T \subseteq \states$. If $m>0$ and a pure memoryless policy for the agent $\agentpol$ exists, such that for every state $s \in \states$, we have $\val^{s,\agentpol}_{\rmdp}\left(\reach(T)\right)\geq m$, then for every state $s \in \states$, we have $\val^{s,\agentpol}_{\rmdp}\left(\reach(T)\right)=1$.
\end{restatable}

In case $B \neq \varnothing$, the algorithm computes $Z = \PA^\rmdp_\env(B)$, which are the states where the environment can employ a policy to reach $B$ with non-zero probability, hence reaching $T$ with probability less than $1$. Note that by removing $Z$ from $\rmdp$, the algorithm also updates $\actions(s)$ for all $s \in \states$ to only contain those actions that cannot cause $Z$ being reached in one step. 
The algorithm removes $Z$ from $\rmdp$ and carries on to the next loop iteration. The following theorem establishes the correctness and oracle complexity of the algorithm and shows that almost-sure and limit-sure coincide for reachability objectives.

\begin{restatable}{theorem}{theoremreach}
  \label{thm:reach}
    Given an RMDP $\rmdp = (\states, \actions, \uncert)$ and a target set $T \subseteq \states$, the following assertions hold:
    \begin{compactenum}
        \item (Correctness) for all states $s \in \states$, we have $\val^{s}_{\rmdp}\left(\reach(T)\right)=1$, if and only if $s \in \asreach(\rmdp,T)$;
        \item (Oracle Complexity) $\asreach(\rmdp,T)$ always terminates with $O(|\states|^3)$ oracle calls; and
        \item (Almost-Sure vs Limit-Sure) there exists a pure memoryless policy $\agentpol^*$ such that for all states $s \in \states$, we have $\val^{s}_{\rmdp}\left(\reach(T)\right)=1$, if and only if $\val^{s,\agentpol^*}_{\rmdp}\left(\reach(T)\right)=1$.
    \end{compactenum}
\end{restatable}

\begin{algorithm}[tb]
   \caption{$\asreach(\rmdp,T)$ Almost-Sure Reachability}
   \label{alg:reach}
\begin{algorithmic}[1]
   \STATE {\bfseries Input:} RMDP $\rmdp=(\states,\actions, \uncert)$, Target set $T$.
   \REPEAT
   \STATE $B = \states \setminus \PA_\agent(\rmdp,T)$   \COMMENT{cannot reach $T$}
   \STATE $Z = \PA_\env(\rmdp,B)$ \COMMENT{cannot \em{almost-surely} reach $T$}
   \STATE $\rmdp = \rmdp \setminus Z$
   \UNTIL{$B = \varnothing$}
   \STATE \textbf{return} $\states$
\end{algorithmic}
\end{algorithm}

The following example shows why the loop in Algorithm \ref{alg:reach} is necessary.

\paragraph{Example.} Consider the RMDP in Figure \ref{fig:reach-example} with the uncertainty set as in the previous example. Suppose $\{s_5\}$ is the target set. In its first iteration, the algorithm computes $B=\states \setminus \PA_\agent(\{s_5\}) = \{s_4\}$. This is because by taking $a$ from $s_1,s_2$ and $s_3$, there is a non-zero probability of reaching $s_5$ in the next step. Next, the algorithm computes $Z = \PA_\env(\rmdp,B) = \{s_3,s_4\}$. The reason is that the only available action at $s_3$ is $a$ which has a non-zero probability of reaching $s_4$. The algorithm then removes $Z$ from $\rmdp$ and proceeds to the next loop iteration, where only $s_2$ is removed. In the third loop iteration $s_1$ will be removed. The algorithm will then terminate by returning $\{s_5\}$.

This example shows the existence of RMDPs where the loop in Algorithm \ref{alg:reach} is taken more than once. By extending the chain of $s_i$ states in the example, one can construct an RMDP, for which Algorithm \ref{alg:reach} takes $O(|\states|)$ number of iterations before terminating. 


\subsection{Parity}
In this section, we introduce our algorithm to solve almost-sure parity for the agent, inspired by the algorithm of~\cite{Zielonka98}. The outline of our approach is formally shown in Algorithm~\ref{alg:parity-zielonka}. In addition, we present an analogous algorithm to solve almost-sure parity for the environment, shown in Appendix~\ref{app:asparity-env}. Complete proofs of this section are provided in Appendix~\ref{app:thm:parity}.

Given an RMDP $\rmdp$ with state-set $\states$, a priority function $c$, and maximum priority $d$, the algorithm proceeds in steps, removing several states in each step. Without loss of generality, we assume that $d$ is even (if $d$ is odd, we increase it by 1). Each step first computes the set $B = \states \setminus \PA_\agent(\rmdp,\states_d)$. These are states of the RMDP in which the environment can employ a policy to never reach the set of maximum priority states $\states_d$. It then constructs a new RMDP $\rmdp'$ by restricting the state-space to $B$. Note that by restricting the state-space to $B$ in $\rmdp$, the environment only chooses transitions $\trans \in \uncert(s,a)$ if $\trans[B] = 1$. Next, it recursively computes the set of states $W_\env$ in $\rmdp'$, where the environment can employ a policy to guarantee almost-surely that the maximum priority visited infinitely often is odd. Finally, it removes the states in $\PA_\env(\rmdp,W_\env)$ from $\rmdp$ and carries on to the next loop iteration. The intuition behind the algorithm is as follows. Every run starting from the state $s \in \PA_\agent(\rmdp,\states_d)$ either stays in $\PA_\agent(\rmdp,\states_d)$ forever or eventually reaches $B$. If it always stays in $\PA_\agent(\rmdp,\states_d)$, then the agent has a policy of reaching $\states_d$ infinitely often, and consequently, almost surely satisfies the parity objective. Otherwise, it eventually reaches $B$. The environment does not have any incentive to visit $\PA_\agent(\rmdp,\states_d)$ again and the agent cannot enforce it. Therefore, the run always stays in $B$. 
If $W_\env = \varnothing$, then no matter what policy the environment chooses, the agent has a policy which guarantees the parity with a non-zero probability. Lemma~\ref{lem:parityisas} shows that in such cases, for all states, the agent has a policy with value 1.

\begin{restatable}{lemma}{lemmaparity}
  \label{lem:parityisas}
    Let $\rmdp=(\states,\actions, \uncert)$ and $c$ be an RMDP and a priority function, respectively. If $m>0$ and a pure memoryless policy for the agent $\agentpol$ exists, such that for every state $s \in \states$, we have $\val^{s,\agentpol}_{\rmdp}\left(\parity(c)\right)\geq m$, then for every state $s \in \states$, we have $\val^{s,\agentpol}_{\rmdp}\left(\parity(c)\right)=1$.
\end{restatable}
If $W_\env \neq \varnothing$, then the environment employs a policy for the states in $\PA_\env(\rmdp,W_\env)$ so that the parity objective is satisfied with probability less than 1. Removing the set $\PA_\env(\rmdp,W_\env)$ from $\rmdp$ results in a new RMDP with a smaller state space. 
The algorithm continues until $W_\env = \varnothing$. The following theorem establishes the correctness and oracle complexity of the algorithm and shows that almost-sure and limit-sure coincide for reachability objectives.

\begin{restatable}{theorem}{theoremparity}
  \label{thm:parity}
    Given an RMDP $\rmdp = (\states, \actions, \uncert)$ and a priority function $c \colon \states \to \{0, \ldots, d\}$, the following assertions~hold:

    \begin{compactenum}
        \item (Correctness) for all states $s \in \states$, we have $\val^{s}_{\rmdp}\left(\parity(c)\right)=1$, if and only if $s \in \asparity_\agent(\rmdp,c,d)$;
        \item (Oracle Complexity) $\asparity_\agent(\rmdp,c,d)$ always terminates with $O \left(|\states|^{d+2} \right)$ oracle calls; and
        \item (Almost-Sure vs Limit-Sure) there exists a pure memoryless policy $\agentpol^*$ such that for all states $s \in \states$, we have $\val^{s}_{\rmdp}\left(\parity(T)\right)=1$, if and only if $\val^{s,\agentpol^*}_{\rmdp}\left(\parity(T)\right)=1$.
    \end{compactenum}
\end{restatable}

\begin{remark}
In case the maximum priority is constant, the algorithm~$\asparity_\agent$ runs in polynomial time. Many objectives such as liveness, reach-avoid, safety, B{\"u}chi, co-B{\"u}chi, progress, etc. can be modeled as parity objectives with constant maximum priority. However, the upper bound over oracle calls is exponential with respect to the maximum priority in general. In Appendix~\ref{app:efficient-parity}, we present a similar algorithm $\easparity_\agent$ with a quasi-polynomial upper bound over oracle calls.
\end{remark}

\begin{algorithm}[tb]
   \caption{$\asparity_\agent(\rmdp,c, d)$ Almost-Sure Parity for Agent}
   \label{alg:parity-zielonka}
\begin{algorithmic}[1]
   \STATE {\bfseries Input:} RMDP $\rmdp=(\states,\actions, \uncert)$, priority function~$c$, maximum priority $d$
   \IF{$\states = \varnothing$}
   \STATE \textbf{return} $\varnothing$
   \ENDIF
   
   \REPEAT
   \STATE $\states_d = \{s \in \states | c(s) = d\}$   \COMMENT{maximum priority states}
   \STATE $B = \states \setminus \PA_\agent(\rmdp,\states_d)$   \COMMENT{cannot reach priority-$d$}
   \STATE $\rmdp' = \rmdp_{|B}$
   \STATE $W_\env = \asparity_\env(\rmdp', c, d-1)$ \COMMENT{solve for environment}
   \STATE $G = \PA_\env(\rmdp,W_\env)$ \COMMENT{does not satisfy parity almost-surely for agent}
   \STATE $\rmdp = \rmdp \setminus G$
   \UNTIL{$W_\env \equiv \varnothing$}
   
   \STATE \textbf{return} $\states$
\end{algorithmic}
\end{algorithm}

\begin{restatable}{corollary}{theoremeffparity}
  \label{thm:eff-parity}
    Given an RMDP $\rmdp = (\states, \actions, \uncert)$ and a priority function $c \colon \states \to \{0, \ldots, d\}$, the following assertions~hold:

    \begin{compactenum}
        \item (Correctness) for all states $s \in \states$, we have $\val^{s}_{\rmdp}\left(\parity(c)\right)=1$, if and only if $s \in \easparity_\agent(\rmdp,c,d,|\states|,|\states|)$; and
        \item (Oracle Complexity) $\easparity_\agent(\rmdp,c,d,|\states|,|\states|)$ always terminates with $|\states|^{O \left(\log_2 d \right)}$ oracle calls.
    \end{compactenum}

\end{restatable}

The following example shows how Algorithm~\ref{alg:parity-zielonka} is applied to our running example:
\paragraph{Example.} Consider the RMDP in Figure \ref{fig:reach-example} with the uncertainty set as in the previous examples. The priority function~$c$ is defined as follows.
\[
    c(s) = \begin{cases}
        1  &s \in \{s_2, s_4\}\\
        2  &otherwise
    \end{cases}
\]
First, the procedure $\asparity_\agent(\rmdp, c, 2)$ is called. In its first iteration, the algorithm computes $B=\states \setminus \PA_\agent(\rmdp,\{s_1, s_3, s_5\}) = \{s_4\}$. Next, the algorithm constructs a new RMDP $\rmdp' = \rmdp_{|\{s_4\}}$. The algorithm then calls $\asparity_\env(\rmdp', c, 1)$ which outputs $\{s_4\}$ since there are no states with even priority in this sub-RMDP. The algorithm computes $\PA_\env(\rmdp,\{s_4\}) = \{s_3, s_4\}$. This is because by taking~$a$ from $s_3$, there is a non-zero probability of reaching $s_4$. It then removes $s_3$ and $s_4$ from $\rmdp$ and proceeds to the next loop iteration, where only $s_2$ is removed. The algorithm will then terminate by returning $\{s_1, s_5\}$.



%% file: experiments.tex
\section{Experimental Evaluation} \label{sec:experiments}

We implemented a prototype of our algorithms $\asreach$, $\asparity$ and $\easparity$ to perform experiments and observe their practical performance. We compare the performance of our almost-sure reachability algorithm $\asreach$ against the state-of-the-art policy iteration method (RPPI) on polytopic uncertainty sets \cite{ChatterjeeGK0Z24}. 

The purpose of this section is threefold, (i) to demonstrate the applicability and efficiency of our algorithm for almost-sure reachability ($\asreach$) on well-motivated benchmarks (ii) to show runtime improvement of $\asreach$ on polytopic uncertainty sets ($L_1$ and $L_\infty$) in comparison to the baseline (RPPI), and (iii) to demonstrate and compare the performance of $\asparity$ and $\easparity$ algorithms.

\paragraph{Baselines.}
We compare the runtime of our algorithm $\asreach$ against the {\em robust polytopic policy iteration} (RPPI) algorithm proposed in~\cite{ChatterjeeGK0Z24}. Given a polytopic RMDP, RPPI reduces the problem of computing long-run average value to that of a two-player stochastic game. It then uses policy iteration to find the long-run average value of the game. 

In order to compare $\asreach$ to this baseline, we slightly change our benchmarks by adding a self-loop to the target states. We then set a reward of 0 for all transitions except for the newly added self-loops which are assigned a reward of 1. With this change, the target set is almost-surely reachable from an initial state $s$ if and only if the long-run average value of $s$ (computed by RPPI) is equal to 1.

\input{experiment-tables}

\paragraph{Implementation Details.}
We implement Algorithm~\ref{alg:reach}, Algorithm~\ref{alg:parity-zielonka}, and Algorithm~\ref{alg:parity} to compute the set of states where the value of the RMDP initializing at those states is equal to 1 with respect to the reachability and parity objectives explained next. For computing oracles, we implemented algorithms provided in Appendix~\ref{app:oracles}. The time limit for all experiments is set to 60 seconds. The algorithms are implemented in Python~3.11 and all the experiments were run on a machine with an Apple M1 chip with macOS~14.3, and 8GB of RAM.

\paragraph{Benchmarks.}
We present two sets of benchmarks: Frozen Lake benchmarks with reachability and parity objectives and QComp benchmarks with reachability objectives. For robustness, we consider uncertainty sets with respect to $L_p$ norms between probability distributions for $p=1,2$ and $\infty$. The details are as follows:
\begin{compactitem}
    \item \textbf{Frozen Lake.} This benchmark modifies the Frozen Lake environment in OpenAI Gym~\cite{OpenAIGym}. Consider an $n \times n$ grid with an agent in the top left corner of it. The agent can choose to move right, left, up, or down but some cells in the grid are holes where the agent cannot enter them. Moreover, the grid is slippery, meaning that if the agent chooses to move in one direction, it may move in a perpendicular direction to the chosen action. We make this model robust by introducing uncertainty to the transitions. For the uncertainty, each state $s$ of the model is equipped with a non-negative real number $R(s)$ such that $0 \leq R(s) \leq R_{max}$ for some upper bound $R_{max}$. the adversarial environment can change the probabilities such that the $L_p$ distance of the modified and the original probability distributions is less than or equal to $R(s)$ when the RMDP is in state $s$. As a reachability objective, the goal of the agent is to reach the bottom right cell of the grid. We consider frozen lake models with $n=10k$ for $1 \leq k \leq 8$ and choose the holes and the uncertainty radii $R(s)$ randomly with three different upper bounds ($R_{max} \in  \{0.5,1,1.5\}$). Moreover, we used three different random seeds (for choosing the wholes and the uncertainty radii), leaving us with 24 instances for each configuration of $p$ and $R_{max}$.

    For an $\omega$-regular objective, suppose that there are two types of resources, one provided in the leftmost column and the other on the rightmost column of the grid. The agent's goal is to alternate between these two sets of resources forever. 
    Similar to reachability, we chose the holes and the uncertainty radii $R(s)$ randomly with three different upper bounds.
    In this case, we consider models with $n=10k$ for $1 \leq k \leq 5$ (a total of 15 instances for each configuration of $p$ and $R_{max}$). Smaller models were chosen for the parity objectives due to the fact that our parity algorithms have super-polynomial complexity while the reachability algorithm runs in polynomial time. 

    \item \textbf{QComp.} This benchmark set is inspired by the MDPs provided in the QComp repository~\cite{HartmannsKPQR19}. The original benchmark set contains huge MDPs that require massive resources for analysis. Therefore, we only consider those MDPs that are equipped with reachability objectives and have less than or equal to $10^4$ states. This leaves us with 24 instances. We introduce robustness to these benchmarks by allowing an adversarial environment to modify the probability distributions such that the $L_p$ distance between the modified and the original probability distributions is less or equal to a global value $R$.
\end{compactitem}

\vspace{-1em}
\paragraph{Tables Description.} The results for reachability and parity algorithms are summarized in Table~\ref{table:reach} and Table~\ref{table:parity}, respectively. The parameter $p$ corresponds to the $L_p$ norms used to define the uncertainty sets in the problem, with $p=1,2,\text{or }\infty$, representing different geometric shapes of uncertainty sets. The parameter $R_{max}$ denotes the maximum radius of the uncertainty sets, with $R_{max}=0.5,1, \text{or }1.5$. For each combination of $p$ and $R_{max}$, the following metrics are reported: (a)~count: the number of instances solved within the time limit; and (b)~avg. time: the average time (in seconds) required to solve instances that were successfully solved within the time limit. Higher count values and lower avg. time values indicate better performance.

\vspace{-0.5em}
\paragraph{Results for Reachability.} 
We compared our algorithm, $\asreach$, with the baseline approach, as shown in Table~\ref{table:reach}. The results demonstrate that $\asreach$ outperforms the baseline by solving 426 total instances, compared to only 81 solved by the baseline. This highlights the effectiveness of our algorithm for solving the almost-sure reachability problem, outperforming the RPPI algorithm. Moreover, it is noteworthy that while the RPPI approach is restricted to polytopic RMDPs, Table~\ref{table:reach} shows that our method can be efficiently applied to RMDPs with $L_2$ uncertainty sets. An advantage of our method is that it is oracle-based, hence we expect it to solve RMDPs with higher degree $L_p$-based uncertainties as well. Lastly, note that on the benchmarks that were solved by both methods, the average runtime of our approach is about 0.1 seconds while the baseline takes more than 6 seconds. This shows that while the baseline only solves the simplest benchmarks, our method is faster even on these simple instances by an order of magnitude. 

%
%
\vspace{-0.5em}
\paragraph{Results for Parity.} Table~\ref{table:parity} summarizes our experimental results comparing our two algorithms, $\asparity$ and $\easparity$ for parity objectives over the frozen lake benchmarks. These results show that although the worst-case complexity of $\asparity$ is worse than $\easparity$, but in practice, $\asparity$ is faster and can solve 135 instances in the designated timeout, where $\easparity$ can solve 120.
Moreover, Figure~\ref{fig:parity} shows the number of solved instances over time for each algorithm. The fact that $\asparity$'s curve is higher and more to the left, shows that $\asparity$ performs better on our benchmarks. 


\begin{figure}[t]
	\centering
	\includegraphics[scale=0.4]{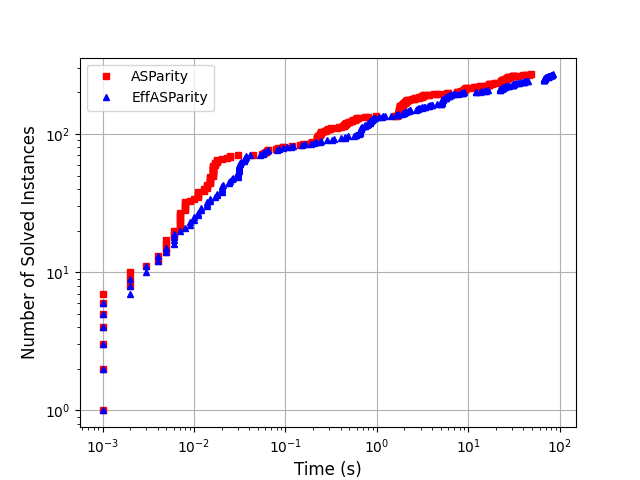}
	\caption{Runtime comparison of our parity algorithms on the Frozen Lake Model. 
    }
	\label{fig:parity}
\end{figure}

%% file: experiment-tables.tex
\begin{table*}[t!]
\centering
\resizebox{\textwidth}{!}{
\begin{tabular}{|c|c||c|c|c|c|c|c||c|c|c|c|c|c|}
\hline
\multirow{3}{*}{Benchmark} & \multirow{3}{*}{Uncertainty} & \multicolumn{6}{|c||}{\textbf{$\asreach$}} & \multicolumn{6}{|c|}{\textbf{$\mathit{RPPI}$}} \\ \cline{3-14}
                                  &                                  & \multicolumn{2}{|c|}{$R_{max}=0.5$} & \multicolumn{2}{|c|}{\textbf{$R_{max}=1$}} & \multicolumn{2}{|c||}{\textbf{$R_{max}=1.5$}} & \multicolumn{2}{|c|}{\textbf{$R_{max}=0.5$}} & \multicolumn{2}{|c|}{\textbf{$R_{max}=1$}} & \multicolumn{2}{|c|}{\textbf{$R_{max}=1.5$}} \\ \cline{3-14}
                                  & & \small{count} & \small{avg. time} & \small{count} & \small{avg. time} & \small{count} & \small{avg. time} & \small{count} & \small{avg. time} & \small{count} & \small{avg. time} & \small{count} & \small{avg. time} \\ \noalign{\hrule height 0.3mm} \hline
\multirow{3}{*}{\textbf{Frozen Lake}} & $L_1$ & 24 &  10.1 &  24 &  10.8 &  24 &  10.0 &  4 &  1.0 &  4 &  5.5 &  1 &  6.2 \\ \cline{2-14}
                             & $L_2$ & 24 &  10.3 &  24 &  0.7 &  24 &  0.3 &  - &  - &  - &  - &  - &  - \\ \cline{2-14}
                             & $L_\infty$ & 24 &  10.6 &  21 &  6.83 &  24 &  0.7 &  4 &  8.2 &  3 &  0.4 &  6 &  2.4 \\ \hline \hline
\multirow{3}{*}{\textbf{QComp}} & $L_1$ & 24 &  6.8 &  23 &  5.1 &  24 &  6.4 &  9 & 7.0 &  10 &  7.8 &  10 &  7.8 \\ \cline{2-14}
                             & $L_2$ & 23 &  5.4 &  24 &  7.4 &  24 &  6.7 &  - &  - &  - &  - &  - &  - \\ \cline{2-14}
                             & $L_\infty$ & 23 &  5.4 &  24 &  6.9 &  24 &  7.0 &  10 &  6.9 &  10 &  6.8 &  10 &  6.8 \\ \hline \hline
\multicolumn{2}{|c||}{Total}  & 142 & 8.1 & 140 & 6.3 & 144 & 5.2 & 27 & 6.3 & 27 & 6.3 & 27 & 6.2 \\
\multicolumn{2}{|c||}{Solved By Both}  & 27 & 0.1 & 27 & 0.1 & 27 & 0.1 & 27 & 6.3 & 27 & 6.3 & 27 & 6.2 \\
\hline
\end{tabular}
}
\caption{Runtime comparison of reachability algorithms on Frozen Lake and QComp benchmarks. The average times are in seconds.}
\label{table:reach}
\end{table*}

\begin{table*}[t!]
\centering
\resizebox{\textwidth}{!}{
\begin{tabular}{|c|c||c|c|c|c|c|c||c|c|c|c|c|c|}
\hline
\multirow{3}{*}{Benchmark} & \multirow{3}{*}{Uncertainty} & \multicolumn{6}{|c||}{\textbf{$\asparity$}} & \multicolumn{6}{|c|}{\textbf{$\easparity$}} \\ \cline{3-14}
                                  &                                  & \multicolumn{2}{|c|}{$R_{max}=0.5$} & \multicolumn{2}{|c|}{\textbf{$R_{max}=1$}} & \multicolumn{2}{|c||}{\textbf{$R_{max}=1.5$}} & \multicolumn{2}{|c|}{\textbf{$R_{max}=0.5$}} & \multicolumn{2}{|c|}{\textbf{$R_{max}=1$}} & \multicolumn{2}{|c|}{\textbf{$R_{max}=1.5$}} \\ \cline{3-14}
                                  & & \small{count} & \small{avg. time} & \small{count} & \small{avg. time} & \small{count} & \small{avg. time} & \small{count} & \small{avg. time} & \small{count} & \small{avg. time} & \small{count} & \small{avg. time} \\ \noalign{\hrule height 0.3mm} \hline
\multirow{3}{*}{\textbf{Frozen Lake}} & $L_1$ & 15 &  7.3 &  15 &  7.4 &  15 &  8.0 &  12 &  7.8 &  12 &  8.0 &  12 &  8.7 \\ \cline{2-14}
                             & $L_2$ & 15 &  8.8 &  15 &  4.8 &  15 &  0.7 &  12 &  9.0 &  15 &  4.6 &  15 &  0.6 \\ \cline{2-14}
                             & $L_\infty$ & 15 &  7.6 &  15 &  10.9 &  15 &  5.6 &  12 &  8.4 &  15 &  11.6 &  15 &  6.5 \\ \hline \hline 
\multicolumn{2}{|c||}{Total} & 45 & 7.9 & 45 & 7.7 & 45 & 4.8 & 36 & 8.4 & 42 & 8.1 & 42 & 5.0 \\ \hline 
\end{tabular}
}
\caption{Runtime comparison of parity algorithms on Frozen Lake benchmarks. The average times are in seconds.}
\label{table:parity}
\end{table*}

%% file: conclu.tex
\vspace{-0.5em}
\section{Conclusion and Future Work} \label{sec:conclu}
In this work, we presented algorithms for the qualitative analysis of RMDPs with reachability and parity objectives. There are several interesting directions for future work. Whether the algorithmic upper bound for parity objectives can be further improved to polynomial time is a major open problem, even for turn-based games. While the present work establishes the theoretical foundations, another interesting direction is exploring efficient heuristics for parity algorithms for further scalability and practical applications.

%% file: appendix.tex
\section{Oracles for Standard Uncertainty Sets} \label{app:oracles}

In this section, we demonstrate the oracles for several well-known uncertainty sets, namely $L_d$ for $d \in \N$, $L_\infty$, linearly defined and polytopic uncertainty sets. 

\subsection{$L_d$ Uncertainty}
Let $R >0$ and suppose the uncertainty sets of RMDP $\rmdp = (\states, \actions, \uncert,\sinit)$ is defined as follows:
\[
\uncert(s,a) = \{\mathbf{p} \in \Delta(\states) \mid \| \mathbf{p}- \mathbf{\bar{p}}_{s,a}\|_d \leq R\}
\]
where $\mathbf{\bar{p}_{s,a}} \in \uncert(s,a)$ is the center of the ambiguity set. Equivalently, the uncertainty set can be written as follows:
\[
\uncert(s,a) = \{\mathbf{p} \in \Delta(\states) \mid \sum_{t \in \states} |\mathbf{p}[t]- \mathbf{\bar{p}}_{s,a}[t]|^d \leq R^d\}
\]
The oracles for this uncertainty set can be computed as follows:
\begin{algorithm}
   \caption{$\force^\rmdp_\agent(s,B)$ for $L_d$ uncertainty set with radius $R$}
   \label{alg:force-agent-Ld}
\begin{algorithmic}[1]
    \IF{$B\equiv\states \wedge \actions(s) \neq \varnothing$ }
    \STATE \textbf{return} $\true$
    \ENDIF
    \FORALL{$a \in \actions(s)$}
    \STATE $c = \bar{p}_{s,a}[B]$ 
    \STATE $k_1 = (\frac{c}{|S|-|B|})^d \times (|S|-|B|)$  \\\COMMENT{cost of increasing $\mathbf{\bar{p}}_{s,a}[t]$ by $\frac{c}{|S|-|B|}$ for $t \notin B$}
    \STATE $k_2 = \sum\limits_{b \in B} \mathbf{\bar{p}}_{s,a}[b]^d$\\ \COMMENT{cost of decreasing $\mathbf{\bar{p}}_{s,a}[b]$ to 0 for $b \in B$}
    \IF{$k_1+k_2>R^d$}
    \STATE \textbf{return} $\true$
    \ENDIF
    \ENDFOR
    \STATE \textbf{return} $\false$
\end{algorithmic}
\end{algorithm}

The intuition behind the above algorithm is as follows. Suppose the agent wants to force the RMDP to reach $B$ from $s$ in one step, i.e. $\force^\rmdp_\agent(s,B)$. Then he should choose an action $a \in \actions(s)$ such that the environment cannot choose $\delta \in \uncert(s,a)$ where $\delta[B]=0$. In order to choose such $\delta$, the environment has to decrease the entries $\bar{\mathbf{p}}_{s,a}[b]$ to 0, for every $b \in B$. This results a distance $k_2$ between $\delta$ and $\bar{\mathbf{p}}_{s,a}$. She then has to distribute $c$ amount of probability among states of $\states \setminus B$. Due to the convexity of the $\|\cdot\|_d$ norm, the optimal choice is to distribute the mass equally among $\states \setminus B$. This incurs a total distance of $k_1$ between $\delta$ and $\bar{\mathbf{p}}_{s,a}$. Therefore, if $\sqrt[d]{k_1+k_2}>R$, the environment cannot choose any member of $\uncert(s,a)$ to guarantee $B$ not being reached in one step, hence $\force^\rmdp_\agent(s,B) \equiv \true$.

In $L_d$ uncertainty sets, the environment can always manipulate the probabilities in $\bar{\mathbf{p}}_{s,a}$, to ensure that $B$ is reached with nonzero probability in one step, which makes the computation of $\force^\rmdp_\env(s, B)$ trivial, that is, $\force^\rmdp_\env(s, B) \equiv \true$. To make the environment less powerful, a similar family of uncertainty sets is studied in the literature~\cite{meggendorfer2024,ChatterjeeGK0Z24,HoPW21} where the uncertainty set does not allow the support of $\bar{\mathbf{p}}_{s,a}$ to expand, i.e.

\[
\begin{split}
\uncert(s,a) = \{\mathbf{p} \in \Delta(\states) \mid &\supp(\mathbf{p})\subseteq \supp(\mathbf{\bar{p}_{s,a}})  \\
&\wedge \| \mathbf{p}- \mathbf{\bar{p}_{s,a}}\|_d \leq R\}
\end{split}
\]
In both cases, both $\force^\rmdp_\agent$ and $\force^\rmdp_\env$ can be computed in polynomial time.

\subsection{$L_\infty$ Uncertainty}

Given a radius $R>0$, the $L_\infty$ uncertainty set is defined similar to $L_d$ uncertainty sets, by replacing the $L_d$ norm $\|\cdot\|_d$ by the $L_\infty$ norm $\|\cdot\|_\infty$ as follows:
\[
\begin{split}
\uncert(s,a) &= \{\mathbf{p} \in \Delta(\states) \mid \| \mathbf{p}- \mathbf{\bar{p}}_{s,a}\|_\infty \leq R\} \\
&= \{\mathbf{p} \in \Delta(\states) \mid \| \max_{t \in \states} |\mathbf{p}[t] - \mathbf{\bar{p}}_{s,a}[t]| \leq R\}
\end{split}
\]

The method for computing $\force^\rmdp_\agent(s,B)$ for $L_\infty$ uncertainty sets is similar to Algorithm \ref{alg:force-agent-Ld}. The details are as in Algorithm \ref{alg:force-agent-Linf}.

\begin{algorithm}
   \caption{$\force^\rmdp_\agent(s,B)$ for $L_\infty$ uncertainty set with radius $R$}
   \label{alg:force-agent-Linf}
\begin{algorithmic}[1]
    \IF{$B\equiv\states \wedge \actions(s) \neq \varnothing$ }
    \STATE \textbf{return} $\true$
    \ENDIF
    \FORALL{$a \in \actions(s)$}
    \STATE $c = \bar{p}_{s,a}[B]$ 
    \STATE $k_1 = \frac{c}{|S|-|B|}$  
    \STATE $k_2 = \max\limits_{b \in B} \mathbf{\bar{p}}_{s,a}[b]$
    \IF{$\max(k_1,k_2)>R$}
    \STATE \textbf{return} $\true$
    \ENDIF
    \ENDFOR
    \STATE \textbf{return} $\false$
\end{algorithmic}
\end{algorithm}

The intuition is similar to that of $L_d$ uncertainties: If the agent chooses action $a \in \actions(s)$, the environment has to decrease the probabilities in $\mathbf{\bar{p}}_{s,a}$ to zero. This incurs a cost of $k_2$ in terms of $L_\infty$ distance. She then has to distribute $c$ units of probability mass among the states of $\states \setminus B$. The optimal choice in this case is equal distribution, which incurs $k_1$ cost.

\subsection{Linearly Defined Uncertainty}
Suppose the uncertainty sets of RMDP $\rmdp$ are defined as follows:
\[
\uncert(s,a) = \{ \mathbf{p} \in \Delta(\states) \mid A_{s,a} \mathbf{p} \leq b_{s,a} \}
\]

where $A_{s,a}$ is a real matrix and $b_{s,a}$ is a vector. Then $\force^\rmdp_\agent(s,B)$ can be computed by checking whether $a \in \actions(s)$ exists such that the following system of linear inequalities has no solution:
\[
\begin{split}
    \sum_{s \in \states} p[s] = 1 \\ 
    A_{s,a} \mathbf{p} \leq b_{s,a} \\ 
    \sum_{b \in B} p[b] = 0  
\end{split}
\]
Furthermore, $\force^\rmdp_\env$ can be computed by checking whether for every action $a \in \actions(s)$ the following system of linear inequalities has a solution:
\[
\begin{split}
    \sum_{s \in \states} p[s] = 1 \\ 
    A_{s,a} \mathbf{p} \leq b_{s,a} \\ 
    \sum_{b \in B} p[b] > 0  
\end{split}
\]

In both cases, the computation can be done in polynomial time by solving a system of linear inequalities \cite{polyLP}. Polytopic uncertainty sets are a special case of linear defined uncertainty sets, hence their $\force$ functions can be computed in polynomial time as well. 

\section{Pseudocode of $\PA_\env$}
\label{app:pattre}
As mentioned in the main body, the algorithm for computing $\PA_\env$ is similar to that of $\PA_\agent$. The details of the algorithm is illustrated in Algorithm \ref{alg:pattr-env}.
\begin{algorithm}
   \caption{$\PA_\env(\rmdp,T)$ Positive Attractor for the environment}
   \label{alg:pattr-env}
\begin{algorithmic}[1]
   \STATE {\bfseries Input:} RMDP $\rmdp$ and Target set $T$.
   \STATE Initialize $T_0 = T$, $i=0$.
   \REPEAT
   \STATE $i = i+1$
   \STATE $T_{i} = T_{i-1} \cup \{s \in \states | \force^\rmdp_\env(s,T_{i-1})\}$
   \UNTIL{$T_i \equiv T_{i-1}$}
   \STATE \textbf{return} $T_i$
\end{algorithmic}
\end{algorithm}

\section{Proof of Lemma \ref{lem:attr}} \label{app:lemm:attr}

For every $s \in \states$ and $B \subseteq \states$ such that $\force^\rmdp_\agent(s,B)$, let $a_{s,B} \in A(s)$ be such that for all $\delta \in \uncert(s,a)$, it holds that $\delta[B]>0$.  Furthermore, let $p_\agent$ be defined as follows:
\[
p_\agent := \min_{\substack{\force^\rmdp_\agent(s,B) \\  \trans \in \uncert(s,a_{s,B})}} \trans[B]
\]
Because of the assumption that the uncertainty set $\uncert(s,a)$ is compact, $p_m$ is well-defined and strictly positive. Specifically, it holds that if $\force^\rmdp_\agent(s,B)$ for some $s \in \states$ and $B \subseteq \states$, then $\trans[B]>0$ for any $\trans \in \uncert(s,a_{s,B})$.

Analogously, for every $s\in \states$ and $B \subseteq \states$ such that $\force^\rmdp_\env(s,B)$, let $\trans_{s,a,B} \in \uncert(s,a)$ be such that $\trans_{s,a,B}=\argmax_{\delta \in \uncert(s,a)} \trans[B]$. Moreover, let $p_\env$ be as follows:
\[
p_\env := \min_{\substack{\force^\rmdp_\env(s,B) \\  a \in A }} \trans_{s,a,B}[B]
\]
similar to the previous case, because of the compactness assumption over $\uncert(s,a)$ for all $s,a$, it is implied that $p_\env$ is well-defined and strictly positive.

\lemmapattr*

\begin{proof}
    \textbf{(Correctness of $\PA_\agent$, $\Rightarrow$)} We prove by induction on $i$ that if $s \in T_i$, then there exists an agent policy $\agentpol$ such that $\val^{s, \agentpol}_{\rmdp}(\reach(T)) \geq p_\agent^i$.
    
    For $i=0$, the set $T_0$ contains all the target states, therefore  for every $s\in T_0$ and agent policy $\agentpol$, it holds that $\val^{s, \agentpol}_{\rmdp}(\reach(T))=1=p_\agent^0$.

    Now consider $T_i$ and let $s \in T_i \setminus T_{i-1}$. By definition it follows that $\force^\rmdp_\agent(s,T_{i-1})=\true$ which implies the existence of action $a_{s,T_{i-1}}$ such that for every $\trans \in \uncert(s,a_{s,T_{i-1}})$, the inequality $\trans[T_{i-1}]\geq p_\agent$ is satisfied. Therefore, if the agent chooses $a_{s,T_{i-1}}$ in the first step, the RMDP will reach $T_{i-1}$ in one step with probability at least $p_\agent$ regardless of the environment's choice. The agent can then follow the policy from $T_{i-1}$ that guarantees reachability of $T$ with probability at least $p_\agent^{i-1}$, implying a total probability of $p_\agent^i$ for reaching the target $T$. 

    \textbf{(Correctness of $\PA_\agent$, $\Leftarrow$)} We prove the contra-positive of the statement: If $s \notin \PA_\agent(\rmdp,T)$, then for every agent policy $\agentpol$, there exists an environment policy $\envpol$ such that $\prob^{\agentpol,\envpol}_{\rmdp}(s)[\reach(T)]=0$.

    Let $B= \states \setminus \PA_\agent(\rmdp,T)$. By definition, for every $s \in B$, it holds that $\force^\rmdp_\agent(s,\PA_\agent(\rmdp,T))=\false$. Therefore, for any action $a\in \actions(s)$ that the agent chooses, the environment can choose $\trans_{s,a} \in \uncert(s,a)$ such that $\trans_{s,a}[\PA_\agent(\rmdp,T)]=0$, or equivalently that $\trans[B]=1$. By always choosing $\trans_{s,a}$, the environment can guarantee that $\PA_\agent(\rmdp,T)$ is never reached, hence the target $T$ is reached with probability 0. 

    \textbf{(Correctness of $\PA_\env$, $\Rightarrow$)} We prove by induction on $i$ that if $s \in T_i$, then for every agent policy $\agentpol$, it holds that $\val^{s, \agentpol}_{\rmdp}(\reach(T)) \geq p_\env^i$.
    
    For $i=0$, the set $T_0$ contains all the target states, therefore  for every $s\in T_0$ and agent policy $\agentpol$, it holds that $\val^{s, \agentpol}_{\rmdp}(\reach(T))=1=p_\env^0$.

    Now consider $T_i$ and let $s \in T_i \setminus T_{i-1}$. By the definition of Algorithm \ref{alg:pattr-env}, it follows that $\force^\rmdp_\env(s,T_{i-1})=\true$. Suppose the RMDP starts its run at $s$, and the agent chooses action $a$ at the beginning. The environment can then choose $\delta_{s,a,T_{i-1}}$ so that the RMDP progresses to $T_{i-1}$ with probability at least $p_\env$. The environment can then follow the policy that guarantees reachability of $T$ with probability at least $p_\env^{i-1}$ from $T_{i-1}$, implying that $T$ is reached with a total probability of at least $p_\env^i$. 

    \textbf{(Correctness of $\PA_\env$, $\Leftarrow$)} We prove the contra-positive of the statement: If $s \notin \PA_\env(\rmdp,T)$, then there exists an agent policy $\agentpol$, such that for every environment policy $\envpol$ $\prob^{\agentpol,\envpol}_{\rmdp}(s)[\reach(T)]=0$.

    Let $B = \states \setminus \PA_\env(\rmdp,T)$. For every $s \in B$, it follows from definitions that $\force^\rmdp_\env(s,\PA_\env(\rmdp,T))=\false$. Therefore, for every $s \in B$, there exists action $a_s \in \actions(s)$ such that despite the environment's choice $\trans \in \uncert(s,a_s)$, it holds that $\trans[\PA_\env(\rmdp,T)]=0$, or equivalently that $\trans[B]=1$. Therefore, by always choosing $a_s$, the agent can guarantee that $\PA_\env(\rmdp,T)$ is never reached. Hence the target $T$ is reached with probability 0. 

    \textbf{(Oracle Complexity)} Each loop iteration removes at least one state from $\rmdp$. Therefore, the procedures terminate after at most $|\states|$ iterations. In each step, the algorithms use at most $|\states|$ oracle calls to compute $T_i$. Hence, the total number of oracle calls is $O(|\states|)$.
\end{proof}

\section{Proofs of Lemma~\ref{lem:reachisas} and  Theorem~\ref{thm:reach}}
\label{app:lem:reachisas,thm:reach}

Given an RMDP $\rmdp=(\states,\actions, \uncert)$ and a target set $T \subseteq \states$, let $\reach_{\leq k}(T)$ define the set of all runs that reach a state in $T$ in less than or equal to $k$ steps: 
    \[
    \reach_{\leq k}(T) = \{s_0,a_0, s_1,a_1, \dots \in \runs_\rmdp | \exists i \leq k: s_i \in T \}
    \]

\lemmareach*

\begin{proof}
    Let $s \in \states$.
    Then $k_s \in \N$ exists such that $p_s = \val^{s,\agentpol}_{\rmdp}\left(\reach_{\leq k_s}(T)\right)> 0$. Let $k=\max\limits_{s \in \states} k_s$ and $p = \min\limits_{s \in \states} p_s$. 
    Following the policy $\agentpol$, in every $k$ steps, the target set $T$ is reached with probability at least $p$. Therefore, $T$ is eventually reached almost-surely, implying that $\val^{s,\agentpol}_{\rmdp}\left(\reach(T)\right)=1$.
\end{proof}

\theoremreach*

\begin{proof}
    \textbf{(Correctness)} We prove the correctness in two parts: \\
    \textbf{Claim 1.} If $s \in \asreach(\rmdp,T)$, then $\val^{s}_{\rmdp}\left(\reach(T)\right)=1$. To see this, let $\rmdp' = \rmdp \setminus Z$ be the RMDP after one loop iteration of the algorithm. Note that $\rmdp'$ is an RMDP where the agent can always stay there and the environment cannot enforce to leave. Therefore, if for some $s \in \states$, it holds that $\val^{s}_{\rmdp'}\left(\reach(T)\right)=1$, then $\val^{s}_{\rmdp}\left(\reach(T)\right)=1$. Lastly, when the algorithm terminates, $\states = \PA_\agent(\rmdp,T)$, by Lemma~\ref{lem:attr}, there exists a pure memoryless policy for the agent $\agentpol$ such that every state of $\states$ can reach $T$ with probability at least $p_\agent^{|\states|}$ by following the policy $\agentpol$. Applying Lemma \ref{lem:reachisas} finishes the proof of the claim.

    \textbf{Claim 2.} If $s \notin \asreach(\rmdp,T)$, then for every agent policy $\agentpol$, $\val^{s,\agentpol}_{\rmdp}\left(\reach(T)\right)<1$. Let $\rmdp$ be any RMDP and consider one iteration of the loop being applied on $\rmdp$. Firstly, note that if $s \in B$, then, due to Lemma \ref{lem:attr}, the environment has a policy $\envpol$ to prevent $\rmdp$ from reaching $T$, starting at $s$. Secondly, if $s \in Z$, then, starting at $s$, the environment has a policy to make $\rmdp$ reach $B$ with non-zero probability, hence $\val^{s,\agentpol}_{\rmdp}\left(\reach(T)\right)<1$ for any agent policy $\agentpol$. This shows that for every state $s$ removed during an iteration of the loop in Algorithm \ref{alg:reach}, $\val^{s,\agentpol}_{\rmdp}\left(\reach(T)\right)<1$.
    
    \textbf{(Oracle Complexity)} In each loop iteration, at least one state is being removed from $\rmdp$, this cannot happen more than $|\states|$ times, therefore the algorithm stops after at most $|\states|$ iterations. Each iteration of the loop calls the $\PA$ functions twice, each $\PA$ call uses the oracles at most $|\states|^2$ times (Lemma~\ref{lem:attr}), therefore the algorithm makes $O(|\states|^3)$ oracle calls in total. 

    \textbf{(Almost-Sure vs Limit-Sure)} In the proof of Claim 1, the pure memoryless policy $\agentpol$ ensures that for all states $s$ such that $\val^{s}_{\rmdp}(\reach(T)) = 1$, we have $\val^{s, \agentpol}_{\rmdp}(\reach(T)) = 1$, which completes the proof.
\end{proof}

\section{Pseudocode of $\asparity_\env$}
\label{app:asparity-env}
\paragraph{Parity Objective of the Environment.} Let $c\colon \states \to \{0, 1, \ldots, d\}$ be a function assigning priorities to states of $\rmdp$. A run $\run$ is included in $\coparity(c)$, if the maximum priority visited infinitely often in $\run$ is {\em odd}:
    \[
    \coparity(c) = \{\run \in \runs_\rmdp | \max\left\{c\left(\inf(\pi)\right)\right\} \text{ is odd}\}.
    \]
Note that parity objectives of the agent and the environment are complementary, i.e., $\parity(c) \cup \coparity(c) = \runs_\rmdp$.

\paragraph{Almost-Sure Parity for Environment.} Given a priority function $c\colon \states \to \N$, find a policy $\envpol^*$ for the environment such that for all agent polices $\agentpol$ we have  $\prob^{\agentpol,\envpol^*}_{\rmdp}(s_0)\left[\coparity(c)\right]=1$, or prove that such policy does not exist. 

As mentioned in the main body, the algorithm for solving almost-sure parity for environment $\asparity_\env$ is similar to that of $\asparity_\agent$. The details of the algorithm are illustrated in Algorithm \ref{alg:parity-zielonka-env}. We assume that the maximum priority $d$ is always odd when the procedure is called (if $d$ is even, we increase it by 1).

\begin{algorithm}
   \caption{$\asparity_\env(\rmdp,c, d)$ Almost-Sure Parity for Environment}
   \label{alg:parity-zielonka-env}
\begin{algorithmic}[1]
   \STATE {\bfseries Input:} RMDP $\rmdp=(\states,\actions, \uncert)$, priority function~$c$, maximum priority $d$
   \IF{$\states = \varnothing$}
   \STATE \textbf{return} $\varnothing$
   \ENDIF
   
   \REPEAT
   \STATE $\states_d = \{s \in \states | c(s) = d\}$   \COMMENT{maximum priority states}
   \STATE $B = \states \setminus \PA_\env(\rmdp,\states_d)$   \COMMENT{cannot reach priority-$d$}
   \STATE $\rmdp' = \rmdp_{|B}$
   \STATE $W_\agent = \asparity_\agent(\rmdp', c, d-1)$ \COMMENT{solve for agent}
   \STATE $G =  \PA_\agent(\rmdp,W_\agent)$ \COMMENT{does not satisfy parity almost-surely for environment}
   \STATE $\rmdp = \rmdp \setminus G$ 
   \UNTIL{$W_\agent \equiv \varnothing$}
   
   \STATE \textbf{return} $\states$
\end{algorithmic}
\end{algorithm}

\section{Proofs of Lemma~\ref{lem:parityisas} and Theorem~\ref{thm:parity}}
\label{app:thm:parity}

In this section, we prove Lemma~\ref{lem:parityisas} and Theorem~\ref{thm:parity}. Lemma~\ref{lem:parityisas} is inspired by~\cite{Chatterjee07}, which proved a similar result on concurrent stochastic games with tail objectives. To prove the lemma, we first introduce L\'evy's zero-one law in Lemma~\ref{app:levys} which is a tool from probability theory and then present lemma~\ref{app:zolawrmdp} as an immediate consequence of lemma~\ref{app:levys} in our setting.

\begin{lemma}[L\'evy's zero-one law]
    \label{app:levys}
    For a probability space $(\Omega, F, P)$, let $\mathcal{H}_1, \dots$ be a sequence of increasing $\sigma$-fields, and $\mathcal{H}_\infty = \sigma(\bigcup_{n} \mathcal{H}_n)$. Then, for all events $\mathcal{A} \in \mathcal{H}_\infty$, we have:

    $$\mathbb{E}[\textbf{1}_\mathcal{A} | \mathcal{H}_\infty] = \prob(\mathcal{A} | \mathcal{H}_\infty) \xrightarrow[]{n \rightarrow \infty} \textbf{1}_\mathcal{A} \text{;   almost-surely}.$$
\end{lemma}



\begin{lemma}[zero-one law in RMDPs]
    \label{app:zolawrmdp}
    Given an RMDP $\rmdp$ and a priority function $c$, for all policies $\agentpol$ and $\envpol$, and for all states $s$, we have:
    $$\prob^{\agentpol,\envpol}_\rmdp(s)[\parity(c) | h_n] \xrightarrow[]{n \rightarrow \infty} 0 \text{ or } 1 \text{;  with probability 1.}$$
\end{lemma}

where $h_n$ is a history up to step $n$.

Intuitively, lemma~\ref{app:zolawrmdp} states that as the length of run tends to~$+\infty$ and the agent gains more information about the history of the RMDP, the probability of satisfying parity converges to either 0 or 1.

\lemmaparity*

\begin{proof}
    For all environment policies $\envpol$ and all histories $h_n = s_0, a_0, \dots ,s_{n}$, we have
    \begin{equation}
        \label{eq:probbiggerthanm}
        \prob^{\agentpol,\envpol}_\rmdp(s_0)[\parity(c)|h_n] \geq m
    \end{equation}
    because the policy $\agentpol$ ensures a value of at least $m$ for the state $s_{n}$ and parity objective is independent of all finite prefixes.

    By Lemma~\ref{app:zolawrmdp}, for all environment policies $\envpol$, we have
    \begin{equation}
    \label{eq:probgoesto0or1}
        \prob^{\agentpol,\envpol}_\rmdp(s_0)[\parity(c)|h_n] \xrightarrow{n \to \infty} \{0, 1\}; \quad \text{a.s.}
    \end{equation}
    where $h = s_0, a_0, \dots ,s_{n}$ is the history up to step $n$. Therefore, 
    by Eq.~\ref{eq:probbiggerthanm} and Eq.~\ref{eq:probgoesto0or1}, we have
    \[
        \prob^{\agentpol,\envpol}_\rmdp(s_0)[\parity(c)|h_n] \xrightarrow{n \to \infty} 1; \quad \text{a.s.}
    \]
    which implies that, for all states $s$, we have  $\val^{s,\agentpol}_{\rmdp}\left(\parity(c)\right) = 1$ and completes the proof.

\end{proof}


\theoremparity*
\begin{proof}
We prove the result for the procedure $\asparity_\agent$ only. The proof for the procedure $\asparity_\env$ is analogous.

\textbf{(Correctness)} We prove the correctness of the algorithm by induction on the maximum priority $d$.

Induction Base Case. For $d = 0$, there is only one priority and any policy has value 1.

Induction Step Case.  We assume that maximum priority $d$ is even when the procedure $\asparity_\agent$ is called (if $d$ is odd, we increase it by 1). 
Assume that the procedure $\asparity_\env$ is correct for the maximum priority $d-1$, i.e., given an RMDP $\rmdp'$ with maximum priority $d-1$, for all states $s \in \asparity_\env(\rmdp', c, d-1)$,
the environment has a policy $\envpol^*$ such that for all agent policies $\agentpol$, $\prob^{\agentpol,\envpol}_{\rmdp'}(s)[\parity(c)] = 0$, or equivalently, $\val^{s, \agentpol}_{\rmdp'}\left(\parity(c)\right) = 0$. 

\textbf{Claim 1.} Assume the induction step hypothesis. Let $G_i$ be the set of states removed from $\rmdp$ in the $i$-th iteration of $\asparity_\agent(\rmdp, c, d)$. Then, there exists a policy for the environment $\envpol^*$ such that for all states $s \in G_i$ and agent policies $\agentpol$, we have 
    \[
        \prob^{\agentpol, \envpol^*}_{\rmdp}(s)\left[\parity(c)\right]<1.
    \]

Let $\rmdp_{i}$, $\rmdp'_{i}$, $B_i$, and $W^i_\env$ be the variables used in the $i$-th iteration of the algorithm. By the induction step hypothesis, for all states $s \in W^i_\env$, the environment has a policy $\envpol^*$ such that for all agent policies $\agentpol$, the probability of satisfying the parity in $\rmdp'_{i}$ is 0, i.e., $\prob^{\agentpol,\envpol^*}_{\rmdp'_{i}}(s)[\parity(c)] = 0$.  Notice that by the definition of $B_i$, for all states $s \in B_i$, the agent does not have any policy to leave $B_i$. Therefore, for all states $s \in W^i_\env$, the probability of satisfying the parity in $\rmdp_{i}$ is also 0.  
Moreover, by Lemma~\ref{lem:attr}, for all states $s \in \PA_\env(\rmdp_{i},W_\env^i)$, the environment has a policy that enforces to reach $W_\env^i$ with a non-zero probability. Hence, $\val^{s,\agentpol}_{\rmdp_i}\left(\parity(c)\right)<1$ for all agent policies $\agentpol$. Note that states in $\PA_\env(\rmdp_{i},W_\env^i)$ are not desirable for the agent since their value is less than 1, and the environment cannot enforce to reach them from the rest of the RMDP. Thus, removing $\PA_\env(\rmdp_{i},W_\env^i)$ from the RMDP does not change the values of other states. Consequently, we have $\val^{s,\agentpol}_{\rmdp}\left(\parity(c)\right)<1$, which completes the proof of the claim. 

\textbf{Claim 2.}
    Assuming the induction step hypothesis, for all the states $s \in \asparity_\agent(\rmdp, c, d)$, the agent has a policy $\agentpol^*$ such that $\val^{s, \agentpol}_{\rmdp}[\parity(c)] = 1$.

    We proceed with the proof by constructing the agent policy which satisfies the parity almost-surely. Let $\states^* = \asparity_\agent(\rmdp, c, d)$. We define 
    $\rmdp^*= (\states^*, \actions^*, \uncert^*)=\rmdp_{|\states^*}$, and $B = \PA_\agent(\rmdp^*, \states_d^*)$. We define the policy  for the agent $\agentpol^*$ as follows. (i) from $V \setminus \states_d^*$, follow the policy derived from positive attractor to $\states^*_d$ (Lemma~\ref{lem:attr}); (ii) from $\states_d^*$ stay in $\rmdp^*$ (such action always exists otherwise this state would be removed in previous iterations), and (iii) from $\states^* \setminus V$ play the policy obtained from the recursive call $W_\env = \asparity_\env(\rmdp'^*, c, d') \equiv \varnothing$. Note that the policy $\agentpol^*$ is pure and memoryless. We claim that the policy $\agentpol^*$ is a policy with value 1 for the agent. Note that with this policy, environment cannot enforce to leave $\states^*$. For every environment policy $\envpol$ and state $s \in \states^*$:
    \begin{align*}
        \prob&^{\agentpol^*,\envpol}_\rmdp(s)[\parity(c)] =  \prob^{\agentpol^*,\envpol}_{\rmdp^*}(s)[\parity(c)] \\
        &=\prob^{\agentpol^*,\envpol}_{\rmdp^*}(s)[\parity(c) | \buchi(V)].\prob^{\agentpol^*,\envpol}_{\rmdp^*}(s)[\buchi(V)]\\
        &+\prob^{\agentpol^*,\envpol}_{\rmdp^*}(s)[\parity(c) | \neg \buchi(V)].\prob^{\agentpol^*,\envpol}_{\rmdp^*}(s)[\neg \buchi(V)] \\
        &=1.\prob^{\agentpol^*,\envpol}_{\rmdp^*}(s)[\buchi(V)]\\ &+\prob^{\agentpol^*,\envpol}_{\rmdp^*}(s)[\parity(c) | \neg \buchi(V)].\prob^{\agentpol^*,\envpol}_{\rmdp^*}(s)[\neg \buchi(V)]\\
        &=1.\prob^{\agentpol^*,\envpol}_{\rmdp^*}(s)[\buchi(V)] +1.\prob^{\agentpol^*,\envpol}_{\rmdp^*}(s)[\neg \buchi(V)]\\
        &=1\,.
    \end{align*}
    where $\buchi(V)$ is the set of all runs that visit the set $V$ infinitely often. First equality follows from law of total probability. The second equality follows from  $\prob^{\agentpol^*,\envpol}_{\rmdp^*}(s)[\parity(c) | \buchi(V)] = 1$ because if a run visits the set $V$ infinitely often, it visits the even maximum priority $d$ infinitely often almost-surely as well, due to Lemma~\ref{lem:attr}. The third inequality follows from $\prob^{\agentpol^*,\envpol}_{\rmdp^*}(s)[\parity(c) | \neg \buchi(V)] > 0$, by the induction hypothesis, and $\prob^{\agentpol^*,\envpol}_{\rmdp^*}(s)[\parity(c) | \neg \buchi(V)] = 1$ by Lemma~\ref{lem:parityisas}.

    Recall $\val^{s_0,\agentpol}_\rmdp(\parity(c)) = \inf_{\envpol} \prob^{\agentpol,\envpol}_\rmdp(s_0)[\parity(c)]$, so, $\val^{s,\agentpol^*}_\rmdp(\parity(c))=1$ for all $s \in \states^*$, which completes the proof of the claim.

Claim 1 and Claim 2 imply the induction step, which completes the proof of correctness.

\textbf{(Oracle Complexity)} Each loop iteration removes at least one state from $\rmdp$. Therefore, the algorithm terminates after at most $|\states|$ iterations. Let $F(d)$ define the number of oracle calls of the algorithm with maximum priority $d$. Observe that $F(0) = 0$ and $F(d) = |\states| \times (2|\states|^2 + F(d-1))$. Therefore, for all $d\geq 0$, the algorithm performs $F(d) \leq 2\sum_{i=3}^{d+2}|\states|^i = O(|\states|^{d+2})$ oracle calls.

\textbf{(Almost-Sure vs Limit-Sure)} In the proof of Claim 2, the constructed policy $\agentpol^*$ is pure and memoryless and ensures that for all states $s$ such that $\val^{s}_{\rmdp}(\parity(c)) = 1$, we have $\val^{s, \agentpol^*}_{\rmdp}(\parity(c)) = 1$, which completes the proof.

\end{proof}

\section{Efficient Parity Algorithm}
\label{app:efficient-parity}
In this section, we present an algorithm for solving almost-sure parity with a quasi-polynomial upper bound over oracle calls, inspired by the algorithm of~\cite{Parys19}. Recall the algorithm $\asparity_\agent$ that is presented in main body. This algorithm is an iterative and recursive procedure, where in each step, it removes a subset of states from the RMDP. Since every state is removed at most once from the RMDP, there is at most one step where the size of $W_\env$ (the environment winning set) is larger than $\lfloor |\states|/2 \rfloor$ (half of the number of states). We leverage this observation in our improved algorithm $\easparity_\agent$. We add two new parameters $\size_\agent$ and $\size_\env$. The parameter $\size_\agent$ (resp. $\size_\env$) is an upper bound of the output size for the procedure $\easparity_\agent$ (resp. $\easparity_\env$). We allow the procedures to return empty sets if the output size is larger than the given upper bound. The formal outline of the algorithm for solving almost-sure parity for the agent and the environment is illustrated in Algorithm~\ref{alg:parity} and Algorithm~\ref{alg:parity-env}, respectively. 

\theoremeffparity*

\begin{proof}
    \textbf{(Correctness)} The algorithm $\asparity_\agent$, presented in the main body, is an iterative and recursive procedure, where in each step, it removes a subset of states from the RMDP. Since every state is removed at most once from the RMDP, there is at most one step where the size of $W_\env$ (the environment winning set) is larger than $\lfloor |\states|/2 \rfloor$ (half of the number of states). Since the algorithm $\easparity$ is the algorithm $\asparity$ with upper bounds on the output size, the correctness of algorithm $\easparity$ follows from Theorem~\ref{thm:parity}.

    \textbf{(Oracle Complexity)} Each loop iteration removes at least one state from $\rmdp$. Therefore, the algorithm terminates after at most $|\states|$ iterations. Let $F(d, l)$ define the number of oracle calls of the algorithm with maximum priority $d$ and $l = \lfloor \log_2 (\size_\agent + 1) \rfloor + \lfloor \log_2 (\size_\env + 1) \rfloor$. Observe that $F(0, l) = F(d, 0) = 0$. For all $d, l \geq 1$, we have
    \begin{align}
        F(d, l) &= |\states|(2|\states|^2 + F(d-1, l-1)) + 2|\states|^2 + F(d-1, l)\nonumber\\
        &\le 4|\states|^3 + |\states|F(d-1, l-1) + F(d-1, l) \,.
        \label{eq:recursive-ineq}
    \end{align}
    We now prove, by induction, that $F(d, l) \le 4n^{l+2} \binom{d + l}{l} - 4n^2$. For $d = 0$ or $l = 0$, the inequality holds. For $d, l \geq 1$, we have

    \begin{align*}
        F(d, l) &\leq 4|\states|^3 + |\states|F(d-1, l-1) + F(d-1, l)\\
        &\leq 4|\states|^3 + |\states|^{l+2}\binom{d+l-2}{l-1} - 4|\states|^3\\
        &+ |\states|^{l+2}\binom{d+l-1}{l} - 4|\states|^2\\
        &= |\states|^{l+2} \left( \binom{d+l-2}{l-1} + \binom{d+l-1}{l} \right ) - 4|\states|^2\\
        &\leq |\states|^{l+2} \binom{d+l}{l} - 4|\states|^2\,.
    \end{align*}
    The first inequality follows from Eq.~\ref{eq:recursive-ineq}. The second inequality follows from replacing $F$ with the induction hypothesis. The first equality follows from rearranging the terms. The third inequality follows from $\binom{d+l-2}{l-1} \leq \binom{d+l-1}{l-1}$ and Pascal's rule. Therefore, the induction hypothesis is proven. Consequently, the number of oracle calls when the procedure $\easparity(\rmdp, c, d, |\states|, |\states|)$ is called is 
    \[
    O \left (|\states|^{2\log_2 d} \binom{d + 2\log_2 d}{2\log_2 d} \right ) = |\states|^{O(\log_2 d)}\,,
    \]
    which completes the proof.
\end{proof}

\begin{algorithm}
   \caption{$\easparity_\agent(\rmdp,c, d, \size_\agent, \size_\env)$ Efficient Almost-Sure Parity for Agent}
   \label{alg:parity}
\begin{algorithmic}[1]
   \STATE {\bfseries Input:} RMDP $\rmdp=(\states,\actions, \uncert)$, priority function $c$, maximum priority $d$, maximum size of winning set for agent $\size_\agent$ and environment $\size_\env$.
   \IF{$\states \equiv \varnothing \lor \size_\agent \leq 0$}
   \STATE \textbf{return} $\varnothing$
   \ENDIF
   
   \textcolor{blue}{\texttt{// keep removing the environment winning sets of size at most $\lfloor \frac{\size_\env}{2} \rfloor$}}
   \REPEAT
   \STATE $\states_d = \{s \in \states | c(s) = d\}$   \COMMENT{maximum priority states}
   \STATE $B = \states \setminus \PA_\agent(\states_d)$   \COMMENT{cannot visit priority $d$}
   \STATE $\rmdp' = \rmdp_{|B}$
   \STATE $W_\env = \easparity_\env(\rmdp', c, d-1, \lfloor \frac{\size_\env}{2} \rfloor, \size_\agent)$ \COMMENT{solve for environment}
   \STATE $\rmdp = \rmdp \setminus \PA_\env(W_\env)$ \COMMENT{remove environment wining set}
   \UNTIL{$W_\env \equiv \varnothing$}
   
   \textcolor{blue}{\texttt{// remove at most one environment winning set of size larger than $\lfloor \frac{\size_\env}{2} \rfloor$}}
   \STATE $\states_d = \{s \in \states | c(s) = d\}$
   \STATE $B = \states \setminus \PA_\agent(\states_d)$
   \STATE $\rmdp' = \rmdp_{|B}$
   \STATE $W_\env = \easparity_\env(\rmdp', c, d-1, \size_\env, \size_\agent)$ \COMMENT{solve for environment}
   \STATE $\rmdp = \rmdp \setminus \PA_\env(W_\env)$
   
   \textcolor{blue}{\texttt{// keep removing the environment winning sets of size at most $\lfloor \frac{\size_\env}{2} \rfloor$}}
   \REPEAT
   \STATE $\states_d = \{s \in \states | c(s) = d\}$
   \STATE $B = \states \setminus \PA_\agent(\states_d)$
   \STATE $\rmdp' = \rmdp_{|B}$
   \STATE $W_\env = \easparity_\env(\rmdp', c, d-1, \lfloor \frac{\size_\env}{2} \rfloor, \size_\agent)$ \COMMENT{solve for environment}
   \STATE $\rmdp = \rmdp \setminus \PA_\env(W_\env)$
   \UNTIL{$W_\env \equiv \varnothing$}
   \STATE \textbf{return} $\states$
\end{algorithmic}
\end{algorithm}

\begin{algorithm}
   \caption{$\easparity_\env(\rmdp,c, d, \size_\env, \size_\agent)$ Efficient Almost-Sure Parity for Environment}
   \label{alg:parity-env}
\begin{algorithmic}[1]
   \STATE {\bfseries Input:} RMDP $\rmdp=(\states,\actions, \uncert)$, priority function $c$, maximum priority $d$, maximum size of winning set for agent $\size_\agent$ and environment $\size_\env$.
   \IF{$\states \equiv \varnothing \lor \size_\env \leq 0$}
   \STATE \textbf{return} $\varnothing$
   \ENDIF
   
   \textcolor{blue}{\texttt{// keep removing the agent winning sets of size at most $\lfloor \frac{\size_\agent}{2} \rfloor$}}
   \REPEAT
   \STATE $\states_d = \{s \in \states | c(s) = d\}$   \COMMENT{maximum priority states}
   \STATE $B = \states \setminus \PA_\env(\states_d)$   \COMMENT{cannot visit priority $d$}
   \STATE $\rmdp' = \rmdp_{|B}$
   \STATE $W_\agent = \easparity_\agent(\rmdp', c, d-1, \lfloor \frac{\size_\agent}{2} \rfloor, \size_\env)$ \COMMENT{solve for agent}
   \STATE $\rmdp = \rmdp \setminus \PA_\agent(W_\agent)$ \COMMENT{remove the agent wining set}
   \UNTIL{$W_\agent \equiv \varnothing$}
   
   \textcolor{blue}{\texttt{// remove at most one agent winning set of size larger that $\lfloor \frac{\size_\agent}{2} \rfloor$}}
   \STATE $\states_d = \{s \in \states | c(s) = d\}$
   \STATE $B = \states \setminus \PA_\env(\states_d)$
   \STATE $\rmdp' = \rmdp_{|B}$
   \STATE $W_\agent = \easparity_\agent(\rmdp', c, d-1, \size_\agent, \size_\env)$ 
   \STATE $\rmdp = \rmdp \setminus \PA_\agent(W_\agent)$
   
   \textcolor{blue}{\texttt{// keep removing the agent winning sets of size at most $\lfloor \frac{\size_\agent}{2} \rfloor$}}
   \REPEAT
   \STATE $\states_d = \{s \in \states | c(s) = d\}$
   \STATE $B = \states \setminus \PA_\env(\states_d)$
   \STATE $\rmdp' = \rmdp_{|B}$
   \STATE $W_\agent = \easparity_\agent(\rmdp', c, d-1, \lfloor\frac{\size_\agent}{2}\rfloor, \size_\env)$ 
   \STATE $\rmdp = \rmdp \setminus \PA_\agent(W_\agent)$
   \UNTIL{$W_\agent \equiv \varnothing$}
   \STATE \textbf{return} $\states$
\end{algorithmic}
\end{algorithm}